\documentclass[twoside]{article}

\usepackage[accepted]{aistats2021}
\usepackage[round]{natbib}

\usepackage{url}
\usepackage{amsthm,amssymb}
\usepackage{graphicx}
\usepackage{xspace}
\usepackage{xcolor}
\usepackage{algorithm}
\usepackage{algorithmic}

\theoremstyle{plain}
\newtheorem{thm}{Theorem}
\newtheorem{lem}[thm]{Lemma}
\newtheorem*{lem*}{Lemma}

\theoremstyle{definition}

\newcommand{\R}{\mathbb R}
\newcommand{\bA}{\mathbf A}
\newcommand{\bB}{\mathbf B}
\newcommand{\bM}{\mathbf M}
\newcommand{\bK}{\mathbf K}
\newcommand{\bC}{\mathbf C}
\newcommand{\bI}{\mathbf I}

\newcommand{\bX}{\mathbf X}
\newcommand{\bx}{\mathbf x}
\newcommand{\bb}{\mathbf b}
\newcommand{\ba}{\mathbf a}
\newcommand{\bc}{\mathbf c}

\newcommand{\eps}{\varepsilon}
\newcommand{\cN}{\mathcal N}
\newcommand{\cR}{\mathcal R}
\newcommand{\cB}{\mathcal B}
\newcommand{\cV}{\mathcal V}
\newcommand{\iZ}{\textbf{\em Z}}

\newcommand{\FDRR}{\textsc{\sffamily \small FDrr}\xspace}
\newcommand{\RFDRR}{\textsc{\sffamily \small RFDrr}\xspace}
\newcommand{\CSRR}{\textsc{\sffamily \small CSrr}\xspace}
\newcommand{\RPRR}{\textsc{\sffamily \small RPrr}\xspace}
\newcommand{\RR}{\textsc{\sffamily \small rr}\xspace}
\newcommand{\NOFDRR}{\textsc{\sffamily \small 2LFDrr}\xspace}
\newcommand{\iSVDRR}{\textsc{\sffamily \small iSVDrr}\xspace}

\newcommand{\LR}{{\sffamily \small LR}\xspace}
\newcommand{\HR}{{\sffamily \small HR}\xspace}
\newcommand{\TEMP}{{\sffamily \small TEMP}\xspace}

\begin{document}

%

%

\twocolumn[

\aistatstitle{A Deterministic Streaming Sketch for Ridge Regression}

\aistatsauthor{ Benwei Shi \And Jeff M. Phillips }

\aistatsaddress{ University of Utah \And  University of Utah } ]

\begin{abstract}
We provide a deterministic space-efficient algorithm for estimating ridge regression.
For $n$ data points with $d$ features and a large enough regularization parameter,
we provide a solution within $\eps$ L$_2$ error using only $O(d/\eps)$ space.
This is the first $o(d^2)$ space deterministic streaming algorithm with
guaranteed solution error and risk bound for this classic problem.
The algorithm sketches the covariance matrix by variants of Frequent Directions,
which implies it can operate in insertion-only streams and a variety of distributed data settings.
In comparisons to randomized sketching algorithms on synthetic and real-world datasets,
our algorithm has less empirical error using less space and similar time.
\end{abstract}

\section{\uppercase{Introduction}}

Linear regression is one of the canonical problems in machine learning.
Given $n$ pairs $(\ba_i, b_i)$ with each $\ba_i \in \R^d$ and $b \in \R$,
we can accumulate them into a matrix $\bA \in \R^{n \times d}$ and vector $\mathbf \bb \in \R^n$.
The goal is to find $\bx_0 = \arg\min_{\bx \in \R^d}  \|\bA \bx - \bb \|_2^2$.
It has a simple solution $\bx_0 = \bA^\dagger \bb$ where $\bA^\dagger$ is the pseudoinverse of $\bA$.
The most common robust variant, ridge regression~\citep{HK70},
uses a regularization parameter $\gamma > 0$ to add a squared $\ell_2$ regularizer on $\bx$.
Its goal is
\[
\bx_\gamma = \arg\min_{\bx \in \R^d} \left( \|\bA \bx - \bb \|^2 + \gamma \|\bx\|^2 \right).
\]
This also has simple solutions as
\[\bx_\gamma = \left\{\begin{array}{ll}
(\bA^\top \bA + \gamma \bI)^{-1} \bA^\top \bb, & \text{when $n \ge d$,}\\
\bA^\top (\bA \bA^\top + \gamma \bI)^{-1} \bb, & \text{when $n \le d$,}
\end{array}\right. \]
where $\bI$ is the identity matrix.
The regularization and using $\gamma \bI$ makes regression robust to noise (by reducing the variance),
improves generalization, and avoids ill-conditioning.  

However, this problem is difficult under very large data settings because the
inverse operation and standard matrix multiplication will take $O(d^3 + nd^2)$ time,
which is $O(nd^2)$ under our assumption $n > d$.
And this can also be problematic if the size of $\bA$, at $O(nd)$ space, exceeds memory.
In a stream this can be computed in $O(d^2)$ space by accumulating
$\bA^\top \bA = \sum_i \ba_i^\top \ba_i$ and $\bA^\top \bb = \sum_i \ba_i^\top b_i$.  

\subsection{Previous Sketches}
As a central task in data analysis, significant effort has gone into improving
the running time of least squares (ridge) regression. 
Most improvements are in the form of sketching methods using projection or sampling.
\citet{Sarlos06RP} initiated the formal study of using Random Projections (RP)
for regression to reduce $n$ dimensions to $\ell$ dimensions (still $\ell > d$)
preserving the norm of the $d$ dimension subspace vectors with high probability.
\citet{CW13} extended this technique to runtime depending on the number-of-non-zeroes,
for sparse inputs, with CountSketch (CS).
In non-streaming settings, the space can be reduced to depend on the rank
$r = \mathrm{rank}(\bA)$ in the place of the full dimension.  
\citet{LDFU13} used a different random linear transform, called SRHT,
and the dependence on the error was improved by \citet{CLL+15}.  

These random linear transform methods need a randomly selected subspace embedding with dimension $\ell$,
and the resulting sketches have size $O(\ell d)$.
In the resulting analysis, the value $\ell$ should be greater than $d$ or (if not streaming) $r$.  
If one strictly adheres to this theory, the large space bounds make the methods
impractical when $d$ is large and/or when requiring a high degree of accuracy
(i.e., with small error parameter $\eps$).   
One could of course still use the above methods to project to a small dimension
with $\ell < d$ (as we do in our experiments), but no guarantees are known.  

\citet{MM18} proposed deterministic but not streaming ridge leverage score sampling. 
\citet{cohen_et_al:2016:ORS,Cohen_et_al:2017:IST} proposed streaming but not
deterministic ridge leverage score sampling, relying on sketching techniques like Frequent Directions.
In particular, their algorithms are strictly more complicated than the ones we will present,
relying on additional randomized steps (ridge leverage score sampling) and
analysis beyond the techniques we will employ.  
In particular, the computation of leverage scores depends on $(\bA^\top \bA + \gamma \bI)^{-1}$,
which is also the key for the solution of ridge regression.
These approaches can provide ``risk'' bounds (defined formally later),
where the expected solution error is bounded under a Gaussian noise assumption.  
Recently, \citet{WGM18} re-analyzed the quality of these previous linear ridge
regression sketches from two related views: the optimization view (errors on
objective function $f(\bx) = \|\bA \bx - \bb\|^2 + n\gamma \|\bx\|^2$) and the
statistics view (bias and variance of the solutions $\bx$), but this work does
not specifically improve the space or streaming analysis we focus on.  

Although some of these sketches can be made streaming, if they use $o(d^2)$
space (so beating the simple $O(d^2)$ approach),
they either do not provably approximate the solution coefficients, or are not streaming.   
And no existing streaming $o(d^2)$ space algorithm with any provable accuracy guarantees is deterministic.   

\subsection{Our Results}
We make the observation, that if the goal is to approximate the solution to ridge regression,
instead of ordinary least squares regression, and the regularization parameter is large enough,
then a Frequent-Directions-based sketch (which only requires a single streaming pass)
can preserve $(1 \pm \eps)$-relative error on the solution parameters with only roughly $\ell = O(1/\eps)$ rows.
Thus it uses only $O(d \ell) = O(d/\eps) = o(d^2)$ space.  
In contrast, streaming methods based on random linear transforms require
$\ell = \Omega(1/\eps^2)$ for similar guarantees. 
We formalize and prove this (see Theorems \ref{thm:fd} and \ref{thm:rfd} for more nuanced statements),
show evidence that this cannot be improved, and demonstrate empirically that
indeed the FD-based sketch can significantly outperform random-projection-based
sketches -- especially in the space/error trade-off.

\section{\uppercase{Frequent Directions}}
\label{sec:FD}

\citet{Liberty:2013:SDM} introduced Frequent Directions (FD), then together
with \citet{GhashamiSICOMP16FD} improved the analysis.
It considers a tall matrix $\bA \in \mathbb R^{n \times d}$ (with $n \gg d$)
row by row in a stream.   It uses limited space $O(\ell d)$ to compute a short
sketch matrix $\mathbf B \in \mathbb R^{\ell \times d}$, such that the
covariance error is relatively small compared to the optimal rank $k$
approximation, $\left\|\bA^\top \bA - \mathbf B^\top \mathbf B\right\|_2 \le \eps \left\|\bA - \bA_k\right\|_F^2$.
The algorithm maintains a sketch matrix $\mathbf B \in \mathbb R^{\ell \times d}$
representing the approximate right singular values of $\bA$, scaled by the singular values.
Specifically, it appends a batch of $O(\ell)$ new rows to $\mathbf B$,
computes the SVD of $\mathbf B$, subtracts the squared $\ell$th singular value
from all squared singular values (or marks down to $0$), and then updates
$\mathbf B$ as the reduced first $(\ell-1)$ singular values and right singular vectors.
After each update, $\mathbf B$ has at most $\ell-1$ rows.  
After all rows of $\bA$, for all $k < \ell$:
\begin{equation}\label{eq:fd-err}
    \left\|\bA^\top \bA - \mathbf B^\top \mathbf B\right\|_2 \le \frac 1 {\ell - k} \left\|\bA - \bA_k\right\|_F^2. 
\end{equation} 

The running time is $O(n d \ell)$ and required space is $O(\ell d)$. 
By setting $\ell = k + 1/\eps$, it achieves $\eps \|\bA - \bA_k\|_F^2$
covariance error, in time $O(nd(k + 1/\eps))$ and in space $O((k + 1/\eps)d)$. 
Observe that setting $\ell = \mathrm{rank}(A)+1$ achieves $0$ error in the form stated above.  

Recently, \citet{Luo2019} proposed Robust Frequent Direction (RFD).
They slightly extend FD by maintaining an extra value 
$\alpha \geq 0$, 
which is half of the sum of all squared $\ell$th singular values.
Adding $\alpha$ back to the covariance matrix results in a more robust solution and less\ error.
For all $0 \leq k < \ell$:
\begin{equation}\label{eq:rfd-err}
\left\|\bA^\top \bA - \mathbf B^\top \mathbf B - \alpha \bI \right\|_2 \le \frac 1 {2(\ell - k)} \left\|\bA - \bA_k\right\|_F^2.
\end{equation}

It has same running time and running space with FD in terms of $\ell$. 
To guarantee the same error, RFD needs almost a factor $2$ fewer rows $\ell = 1/(2 \eps) + k$.  

\citet{HuangPMLR18} proposed a more complicated variant to separate $n$ from $1/\eps$ in the running time. 
The idea is two level sketching: not only sketch $\mathbf B \in \mathbb R^{3k \times d}$,
but also sketch the removed part into $\mathbf Q \in \mathbb R^{1/\eps \times d}$ via sampling.
Note that for a fixed $k$, $\mathbf B$ has a fixed number of rows,
only $\mathbf Q$ increases the number of rows to reduce the error bound,
and the computation of $\mathbf Q$ is faster and more coarse than that of $\mathbf B$.
With high probability, for a fixed $k$,
the sketch $\mathbf B^\top \mathbf B + \mathbf Q^\top \mathbf Q$ achieves the
error in (\ref{eq:fd-err}) in time $O(ndk) + \tilde O(\eps^{-3}d)$ using space $O((k+\eps^{-1})d)$. 
By setting $\ell = 3k + 1/\eps$, the running time is
$O(nkd) + \tilde O((\ell-k)^{3}d)$ and the space is $O(\ell d)$

The Frequent Directions sketch has other nice properties.
It can be extended to have runtime depend only on the number of nonzeros for sparse inputs~\citep{GhashamiKDD16EFD,HuangPMLR18}.  
Moreover, it applies to distributed settings where data is captured from multiple locations or streams.
Then these sketches can be ``merged'' together~\citep{GhashamiSICOMP16FD, ACH+12}
without accumulating any more error than the single stream setting, and extend to other models~\citep{ATTP}.
These properties apply directly to our new sketches.  

\subsection{FD and Ridge Regression}

Despite FD being recognized as the matrix sketch with best space/error
trade-off (often optimal~\citep{GhashamiSICOMP16FD}),
it has almost no provable connections improvements to high-dimensional regression tasks. 
The only previous approach we know of to connect FD to linear regression
(\citep{MM18} via \citep{Cohen_et_al:2017:IST}), uses FD only to make the
stream processing efficient, does not describe the actual algorithm, and then
uses ridge leverage scores as an additional step to connect to ridge regression.  
The main challenge with connecting FD to linear regression is that FD
approximates the high norm directions of $\bA^\top \bA$ (i.e., measured with
direction/unit vector $\bx$ as $\|\bA^\top \bA \bx\|)$, but drops the low norm directions.
However, linear regression needs to recover $\bc = \bA^\top \bb$ times the inverse of $\bA^\top \bA$.
So if $\bc$ is aligned with the low norm part of $\bA^\top \bA$, then FD provides a poor approximation.  
We observe however, that ridge regression with regularizer $\gamma \bI$ ensures
that all directions of $\bA^\top \bA + \gamma \bI$ have norm at least $\gamma$,
regardless of $\bA$ or its sketch $\bB$.

\begin{figure}
	\centering
	\includegraphics[width=1\linewidth]{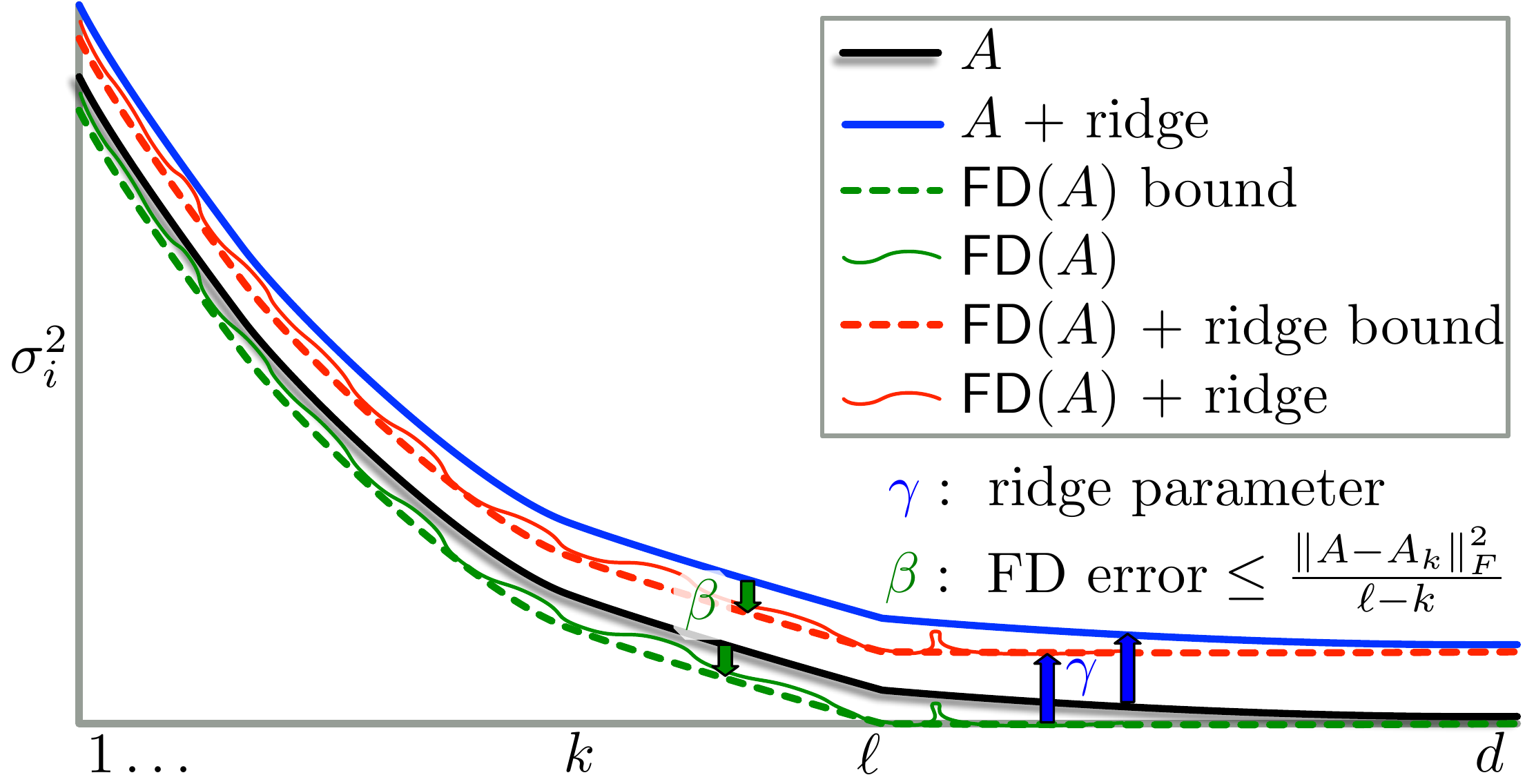}
	\caption[FD+Ridge]{A figurative illustration of possible eigenvalues
	($\sigma_i^2$) of a covariance matrices $\bA^\top \bA$ and variants when
	approximated by FD or adding a ridge term $\gamma\bI$, along sorted eigenvectors.}
	\label{fig:fd-ridge}
\end{figure}

Figure~\ref{fig:fd-ridge} illustrates the effect on the eigenvalue distribution
(as $\sigma_i^2$) for some $\bA^\top \bA$, and how it is affected by a ridge term and FD. 
The ridge term increases the values everywhere, and FD decreases the values everywhere.
In principle, if these effects are balanced just right they should cancel out
-- at least for the high rank part of $\bA^\top \bA$.  
In particular, Robust Frequent Directions attempts to do this implicitly
-- it \emph{automatically picks a good choice of regularizer} $\alpha$ as half
of the amount of the shrinkage induced by FD.  

\section{\uppercase{Algorithms and Analysis}}
We consider rows of $\bA \in \mathbb R^{n \times d}$ and elements of
$\bb \in \mathbb R^{n}$ are given in pairs $(\ba_i, b_i)$ in the stream,
we want to approximate $\bx_\gamma = (\bA^\top \bA + \gamma \bI)^{-1} \bA^\top \bb$
for a given $\gamma>0$ within space $O(\ell d)$, where $\ell < d$.
Let $\mathbf c = \bA^\top \bb$, which can be exactly maintained using space $O(d)$.
But $\bA^\top \bA$ needs space $\Omega(d^2)$, so we use Frequent Directions
(FD) or Robust Frequent Directions (RFD) to approximate $\bA$ by a sketch
(which is an $\ell \times d$ matrix $\mathbf{C}$ and possibly also some auxiliary information).
Then the optimal solution $\bx_\gamma$ and its approximation of $\hat{\bx}_\gamma$ are
\[
\bx_\gamma = (\bA^\top \bA + \gamma \bI)^{-1} \mathbf c \quad \text{ and } \quad \hat{\bx}_\gamma = (\text{sketch} + \gamma \bI)^{-1} \mathbf c.
\]

\begin{algorithm}
	\caption{General FD Ridge Regression (FDRR)}\label{galg}
	\begin{algorithmic}[1]  
		\STATE {\bfseries Input:} $\ell, \bA, \bb, \gamma$
		\STATE Initialize $\textbf{\texttt{xFD}}$, $\mathbf c \gets 0^d$
		\FOR{batches $(\bA_\ell, \bb_\ell) \in \bA, \bb$}
		\STATE $\text{sketch} \gets \textbf{\texttt{xFD}}(\text{sketch}, \bA_\ell)$ \label{galg:xfd}
		\STATE $\mathbf c \gets \mathbf c + \bA_\ell^\top \bb_\ell$ \label{galg:c}
		\ENDFOR
		\STATE $\hat{\bx}_\gamma \gets \mathbf{Solution}(\text{sketch}, \gamma, \mathbf c)$ \label{galg:x}
		\STATE \textbf{return} $\hat{\bx}_\gamma$
	\end{algorithmic}
\end{algorithm}

Algorithm~\ref{galg} shows the general algorithm framework.
It processes a consecutive batch of $\ell$ rows of $\bA$ (denoted $\bA_\ell$)
and $\ell$ elements of $\bb$ (denoted $\bb_\ell$) each step.
\textbf{\texttt{xFD}} refers to a sketching step of some variant of Frequent Directions.
Line 5 computes $\bA^\top \bb$ on the fly, it is not a part of FD.
Line 7 computes the solution coefficients $\hat{\bx}_\gamma$ using only the
sketch of $\bA$ and $\bc$ at the end.
This supplements FD with information to compute the ridge regression solution.

\paragraph{Coefficients error bound.}
The main part of our analysis is the upper bound of the \emph{coefficients error}:
$\eps = \|\hat{\bx}_\gamma - \bx_\gamma\| / \|\bx_\gamma\|$.  
Lemma~\ref{thm:general} shows the key structural result,
translating the sketch covariance error to the upper bound of ridge regression coefficients error.

\begin{lem}\label{thm:general}
Let $\mathbf C^\top \mathbf C$ be an approximation of $\bA^\top \bA \in \mathbb R^{d \times d}$.  
	For any $\mathbf c \in \mathbb R^d$, $\gamma \ge 0$, consider an optimal solution $\bx_\gamma = (\bA^\top \bA + \gamma \bI)^{-1} \mathbf c$, and an approximate  solution $\hat{\bx}_\gamma = (\mathbf C^\top \mathbf C + \gamma \bI)^{-1} \mathbf c$.  Then
	\[
	\|\hat{\bx}_\gamma - \bx_\gamma\| \le \frac{\|\bA^\top \bA - \mathbf C^\top \mathbf C\|_2}{\lambda_{min} (\mathbf C^\top \mathbf C) + \gamma} \|\bx_\gamma\|.  
	\]
\end{lem}

\begin{proof}
	To simplify the equations, let $\mathbf M = \bA^\top \bA + \lambda \bI$, $\hat{\mathbf M} = \mathbf C^\top \mathbf C + \lambda \bI$, then $\mathbf M - \hat{\mathbf M} = \bA^\top \bA - \mathbf C^\top \mathbf C$, and so $\bx_\gamma = \mathbf M^{-1} \mathbf c, \hat{\bx}_\gamma = \hat{\mathbf M}^{-1} \mathbf c$. 
	\begin{align*}
	\|\hat{\bx}_\gamma - \bx_\gamma\| 
	&= \left\|\hat{\mathbf M}^{-1} \mathbf c - \mathbf M^{-1} \mathbf c\right\| 
	= \left\|\left(\hat{\mathbf M}^{-1} - \mathbf M^{-1} \right) \mathbf c\right\|	\\
	&= \left\|\hat{\mathbf M}^{-1} \left(\mathbf M - \hat{\mathbf M} \right) \mathbf M^{-1} \mathbf c\right\|\\
	&\le \left\|\hat{\mathbf M}^{-1}\right\|_2 \left\|\mathbf M - \hat{\mathbf M}\right\|_2 \left\|\mathbf M^{-1} \mathbf c\right\| \\
	&= \frac{\|\bA^\top \bA - \mathbf C^\top \mathbf C\|_2}{\lambda_{\text{min}} (\mathbf C^\top \mathbf C) + \gamma} \|\bx_\gamma\|
	\end{align*}
	The third equality can be validated backwards by simple algebra. Here $\lambda_{\text{min}} (\cdot)$ refer to the minimal eigenvalue of a matrix.
\end{proof}

Lemma \ref{thm:general} is tight when $\bA^\top \bA - \bC^\top \bC = \alpha \bI$, and $\bC^\top \bC = \beta \bI$ for any $\alpha,\beta\in \R$; see Lemma~\ref{thm:general_tight}.  

\begin{lem}\label{thm:general_tight}
	With the same settings as those in Lemma \ref{thm:general},
	if $\bA^\top \bA - \bC^\top \bC = \alpha \bI$,
	and $\bC^\top \bC = \beta \bI$ for any $\alpha,\beta\in \R$, then
	\[
	\|\hat{\bx}_\gamma - \bx_\gamma\| = \frac{\|\bA^\top \bA - \mathbf C^\top \mathbf C\|_2}{\lambda_{min} (\mathbf C^\top \mathbf C) + \gamma} \|\bx_\gamma\|.  
	\]
\end{lem}

\begin{proof}
    In the proof of Lemma \ref{thm:general}, we have shown that
	$\|\hat{\bx}_\gamma - \bx_\gamma\| = \|\hat{\mathbf M}^{-1} \left(\mathbf M - \hat{\mathbf M} \right) \mathbf M^{-1} \mathbf c\|$, 
    Using the definitions $\mathbf M = \bA^\top \bA + \lambda \bI$, $\hat{\mathbf M} = \mathbf C^\top \mathbf C + \lambda \bI$,
	and $\bx_\gamma = (\bA^\top \bA + \gamma \bI)^{-1} \mathbf c$, 
    \begin{align*}
    \|\hat{\bx}_\gamma - \bx_\gamma\| &= \|(\mathbf C^\top \mathbf C + \gamma \bI)^{-1} (\bA^\top \bA - \mathbf C^\top \mathbf C ) \bx_\gamma\|
    \\ &= 
    \|(\beta \bI + \gamma \bI)^{-1} (\alpha \bI) \bx_\gamma\| = \frac{\alpha}{\beta+\gamma} \|\bx_\gamma\|. 
    \end{align*}
    Similarly for the right hand side
    \[
    \frac{\|\bA^\top \bA - \mathbf C^\top \mathbf C\|_2}{\lambda_{min} (\mathbf C^\top \mathbf C) + \gamma} \|\bx_\gamma\| = \frac{\alpha}{\beta+\gamma} \|\bx_\gamma\|. \qedhere
    \]
\end{proof}

\paragraph{Risk bound.} \label{sec:risk_bound}
We consider the fixed design setting commonly used in recent papers~\citep{DFKU13,LDFU13,CLL+15,MM18,WGM18}:
we assume the data generation model is
$
\bb = \bA\bx + s \iZ,
$
where $\bA, \bx$ and $s$ are fixed, $\iZ \sim \cN(\mathbf 0,\bI)$ is the random error.
The \textit{risk} $\cR (\hat \bx)$ of estimator $\hat \bx$ of unknown
coefficient $\bx$ is the expected sum of squared error loss over the randomness of noise,
\[
\cR (\hat \bx) = \mathbb E_\iZ \left[\|\bA \hat\bx - \bA \bx\|^2\right] = \mathbb E_\iZ \left[\|\bA (\hat\bx - \bx)\|^2\right].  
\]
We can further decompose the risk into \textit{squared bias} and \textit{variance},
\begin{align*}
	\cR (\hat \bx) &=  \cB^2 (\hat \bx) + \cV (\hat \bx),\\
	\cB^2 (\hat \bx) &=  \|\bA \left(\mathbb E_\iZ \left[ \hat \bx \right] - \bx\right)\|^2,\\
	\cV (\hat \bx) &=  \mathbb E_\iZ \left[\|\bA \left(\hat\bx - \mathbb E_\iZ \left[ \hat \bx \right]\right)\|^2\right].
\end{align*}

\begin{lem}\label{thm:general_risk}
Given $\bA \in \mathbb R^{n \times d}$, $\bx \in \mathbb R^d$, $s > 0$,
let $\iZ \sim \cN(\mathbf 0,\bI)$ represent the standard Gaussian random variable,
and $\bb = \bA\bx + s \iZ$, let $\mathbf C^\top \mathbf C$ be an deterministic approximation of $\bA^\top \bA$.
Then the risk of optimal ridge regression solution
$\bx_\gamma = (\bA^\top \bA+\gamma \bI)^{-1} \bA^\top \bb$ is the sum of
\begin{align*}
	\cB^2 (\bx_\gamma) &= \gamma^2 \left\Vert \bA (\bA^\top \bA + \gamma \bI)^{-1} \bx \right\Vert^2, \\
	\cV (\bx_\gamma) &= s^2\|\bA (\bA^\top \bA+\gamma \bI)^{-1} \bA^\top\|_F^2.
\end{align*}
The risk of the approximate solution $\hat \bx_\gamma = (\bC^\top \bC + \gamma \bI)^{-1} \bA^\top \bb$
is the sum of
\begin{align*}
    \cB^2 (\hat \bx_\gamma) 
    & =
    \left\Vert\bA \left((\bC^\top \bC + \gamma \bI)^{-1} \bA^\top \bA  - \bI \right)\bx\right\Vert^2 
    \\
    \cV (\hat \bx_\gamma) 
    & = 
    s^2\|\bA (\bC^\top \bC + \gamma \bI)^{-1} \bA^\top\|_F^2
\end{align*}
which are bounded as
\begin{align*}
    \cB^2 (\hat \bx_\gamma) 
    & \leq
    \left(1 + \frac 1 {\gamma^4} \|\bA\|_2^4 \|\bA^\top \bA - \bC^\top \bC\|^2\right) \cB^2(\bx_\gamma)
    \\
    \cV (\hat \bx_\gamma) 
    & \leq
    (1+\|\bA\|_2^2/\gamma)^2 \cV(\bx_\gamma).
\end{align*}
\end{lem}

\begin{proof}
Within this proof, we sometimes use $\bK = \bA^\top \bA$ and $\hat \bK = \bC^\top \bC$ to shorten long equations.

Plugging $\bb = \bA\bx + s \iZ$ into $\bx_\gamma = (\bA^\top \bA+\gamma \bI)^{-1} \bA^\top \bb$ gives us
\begin{align*}
\bx_\gamma = (\bA^\top \bA+\gamma \bI)^{-1} \bA^\top \bA \bx + (\bA^\top \bA+\gamma \bI)^{-1} \bA^\top s \iZ.
\end{align*}

Since the standard Gaussian $\iZ$ is the only random variable and we know that
$\mathbb E_\iZ \left[\bX \iZ\right] = 0$ for any $\bX \in \mathbb R^{d \times n}$,
thus
\begin{align*}
\mathbb E_\iZ \left[\bx_\gamma\right] &= (\bA^\top \bA+\gamma \bI)^{-1} \bA^\top \bA \bx.
\end{align*}
Similarly, we have
\begin{align*}
\hat \bx_\gamma = (\bC^\top \bC + \gamma \bI)^{-1} \bA^\top \bA \bx + (\bC^\top \bC + \gamma \bI)^{-1} \bA^\top s \iZ,
\end{align*}
and
\[
\mathbb E_\iZ \left[\hat \bx_\gamma\right] = (\bC^\top \bC + \gamma \bI)^{-1} \bA^\top \bA \bx.
\]

By definition, the squared bias of $\bx_\gamma$ is
\begin{align*}
\cB^2 (\bx_\gamma) &= \left\Vert\bA \left(\mathbb E_\iZ \left[ \bx_\gamma \right] - \bx\right)\right\Vert^2\\
    &= \left\Vert\bA \left((\bA^\top \bA+\gamma \bI)^{-1} \bA^\top \bA \bx - \bx\right)\right\Vert^2\\
    &= \left\Vert\bA \left((\bK +\gamma \bI)^{-1} \bK  - \bI \right)\bx\right\Vert^2\\
    &= \resizebox{.8\hsize}{!}{$\left\Vert \bA \left( (\bK + \gamma \bI)^{-1} \bK - (\bK+\gamma \bI)^{-1} (\bK + \gamma \bI) \right) \bx \right\Vert^2$}\\
    &= \left\Vert \bA \left(\left(\bK+\gamma \bI\right)^{-1} \left(\bK  - \left(\bK+\gamma \bI\right)\right) \right) \bx \right\Vert^2\\
    &= \left\Vert\bA \left((\bA^\top \bA+\gamma \bI)^{-1} ( - \gamma \bI) \right) \bx \right\Vert^2\\
    &= \gamma^2 \left\Vert \bA (\bA^\top \bA + \gamma \bI)^{-1} \bx \right\Vert^2.
\end{align*}

And the squared bias of $\hat \bx_\gamma$ is
\begin{align*}
    \cB^2 (\hat \bx_\gamma) =& \left\Vert\bA \left(\mathbb E_\iZ \left[ \hat \bx_\gamma \right] - \bx\right)\right\Vert^2\\
    =& \left\Vert\bA \left((\bC^\top \bC + \gamma \bI)^{-1} \bA^\top \bA  - \bI \right)\bx\right\Vert^2.
\end{align*}
By playing with linear algebra, we can show that it is
\begin{align*}
=& \left\Vert\bA \left((\hat\bK + \gamma \bI)^{-1} \bK  - \bI \right)\bx\right\Vert^2\\
=& \resizebox{.97\hsize}{!}{$\left\Vert\bA \left(\left((\hat \bK + \gamma \bI)^{-1} - (\bK + \gamma \bI)^{-1} + (\bK + \gamma \bI)^{-1} \right) \bK  - \bI \right)\bx\right\Vert^2$}\\
=& \resizebox{.97\hsize}{!}{$\left\Vert\bA \left(\left((\hat \bK + \gamma \bI)^{-1} (\bK - \hat \bK) (\bK + \gamma \bI)^{-1} + (\bK + \gamma \bI)^{-1} \right) \bK  - \bI \right)\bx\right\Vert^2$}\\
=& \resizebox{.97\hsize}{!}{$\left\Vert\bA \left((\hat \bK + \gamma \bI)^{-1} (\bK - \hat \bK) (\bK + \gamma \bI)^{-1} \bK + (\bK + \gamma \bI)^{-1} \bK  - \bI \right)\bx\right\Vert^2$}\\
=& \resizebox{.97\hsize}{!}{$\left\Vert\bA \left((\hat \bK + \gamma \bI)^{-1} (\bK - \hat \bK) \bK (\bK + \gamma \bI)^{-1} - \gamma (\bK + \gamma \bI)^{-1} \right)\bx\right\Vert^2$}\\
=& \resizebox{.97\hsize}{!}{$\left\Vert \left(\bA(\hat \bK + \gamma \bI)^{-1} (\bK - \hat \bK) \bA^\top \bA (\bK + \gamma \bI)^{-1} - \gamma \bA (\bK + \gamma \bI)^{-1} \right)\bx\right\Vert^2$}\\
=& \resizebox{.97\hsize}{!}{$\left\Vert \left(\frac 1 \gamma \bA(\hat \bK + \gamma \bI)^{-1} (\bK - \hat \bK) \bA^\top - \bI \right) \gamma \bA (\bK + \gamma \bI)^{-1} \bx\right\Vert^2$}\\
\le& \resizebox{.97\hsize}{!}{$\left\Vert \frac 1 {\gamma} \bA(\hat \bK + \gamma \bI)^{-1} (\bK - \hat \bK) \bA^\top - \bI \right\Vert^2 \gamma^2 \left\Vert\bA (\bK + \gamma \bI)^{-1} \bx\right\Vert^2$}\\
\le& \resizebox{.97\hsize}{!}{$\left(\frac 1 {\gamma^2}  \|\bA\|_2^2 \frac 1 {\gamma^2}  \|\bK - \hat \bK\|^2 \|\bA\|_2^2 + 1\right) \gamma^2 \left\Vert\bA (\bK + \gamma \bI)^{-1} \bx\right\Vert^2$}\\
=& \left(\frac 1 {\gamma^4} \|\bA\|_2^4 \|\bA^\top \bA - \bC^\top \bC\|^2 + 1\right) \cB^2(\bx_\gamma).
\end{align*}
The third equality follows $\hat\bM^{-1} - \bM^{-1} = \hat\bM^{-1}(\bM - \hat\bM)\bM^{-1}$
for any invertable matrices $\bM,\hat\bM$ with the same dimensions,
which has been used in the proof of Lemma \ref{thm:general}. 
The fifth equality follows $(\bK + \gamma \bI)^{-1} \bK  - \bI = - \gamma (\bK + \gamma \bI)^{-1}$,
which has been shown in the derivation of $\cB^2 (\bx_\gamma)$ above.
The last inequality follows $\|(\hat \bK + \gamma \bI)^{-1}\|_2^2 \le \frac{1}{\gamma^2}$
because $\hat \bK$ is positive semi-definite.

For the variance part, by definition, the variance of $\bx_\gamma$ is
\begin{align*}
    \cV (\bx_\gamma) =& \mathbb E_\iZ \left[\|\bA \left(\bx_\gamma - \mathbb E_\iZ \left[ \bx_\gamma \right]\right)\|^2\right]\\
    =& \mathbb E_\iZ \left[\|\bA \left((\bA^\top \bA+\gamma \bI)^{-1} \bA^\top s \iZ\right)\|^2\right]\\
    =& s^2\|\bA (\bA^\top \bA+\gamma \bI)^{-1} \bA^\top\|_F^2.
\end{align*}

And the variance of $\hat \bx_\gamma$ is
\begin{align*}
    \cV (\hat \bx_\gamma) =& \mathbb E_\iZ \left[\|\bA \left(\hat \bx_\gamma - \mathbb E_\iZ \left[ \hat \bx_\gamma \right]\right)\|^2 \right]\\
    =& \mathbb E_\iZ \left[\|\bA \left((\bC^\top \bC + \gamma \bI)^{-1} \bA^\top s \iZ\right)\|^2 \right]\\
    =& s^2\|\bA (\bC^\top \bC + \gamma \bI)^{-1} \bA^\top\|_F^2\\
    =& s^2\|(\bA^\dagger)^\dagger (\bC^\top \bC + \gamma \bI)^{\dagger} ((\bA^\top)^\dagger)^\dagger\|_F^2\\
    =& s^2\|((\bA^{\top})^{\dagger}\bC^\top \bC \bA^{\dagger} + \gamma (\bA^\top)^\dagger \bA^\dagger )^\dagger\|_F^2\\
    \le& s^2\|(\gamma (\bA \bA^\top)^{\dagger})^{\dagger}\|_F^2 = \frac{1}{\gamma^2} s^2 \|\bA^\top \bA\|_F^2\\
    \le& \left(\frac{\|\bA\|_2^2 + \gamma}{\gamma}\right)^2 \cV(\bx_\gamma)\\
    =& (1+\|\bA\|_2^2/\gamma)^2 \cV(\bx_\gamma).
\end{align*}
The fifth equality need the assumption that $\bA$ has full column rank.
The last inequality holds because
\begin{align*}
& \cV (\bx_\gamma) = s^2 \sum_{i=1}^{d} \left(\frac{\sigma_i^2}{\sigma_i^2 + \gamma}\right)^2 \ge s^2 \sum_{i=1}^{d} \left(\frac{\sigma_i^2}{\sigma_1^2 + \gamma}\right)^2\\
= & \frac{s^2}{(\|\bA\|_2^2 + \gamma)^2} \sum_{i=1}^{d} \sigma_i^4 = \frac{s^2}{(\|\bA\|_2^2 + \gamma)^2} \|\bA^\top \bA\|_F^2.
\end{align*}
Here $\sigma_i$ represent the $i$th singular value of $\bA$.
\end{proof}

Note that the variance bound is independent of $\bC^\top \bC$;
this is because it is positive definite and constructed deterministically.
We also get some other variance bounds, Lemma \ref{thm:var_bound_2}
and \ref{thm:var_bound_3} in the the Supplement Materials,
which are related to the spectral bound, but can be much worse when
$\|\bA^\top \bA - \bC^\top \bC\|_2^2 \ne 0$.

\subsection{Using Frequent Directions}

Now we consider Algorithm~\ref{fd-alg} (\FDRR), using FD as \textbf{\texttt{xFD}} in Algorithm~\ref{galg}.
Specifically, it uses the Fast Frequent Directions algorithm \citep{GhashamiSICOMP16FD}.
We explicitly store the first $\ell$ singular values $\mathbf \Sigma$ and
singular vectors $\mathbf V^\top$, instead of $\mathbf B$, to be able to compute the the solution efficiently.
Note that in the original FD algorithm, $\mathbf B = \mathbf \Sigma_\ell \mathbf V_\ell^\top$.
Line 4 and 5 are what FD actually does in each step.
It appends new rows $\bA_\ell$ to the current sketch
$\mathbf \Sigma_\ell \mathbf V_\ell^\top$, calls \textsc{svd} to calculate the
singular values $\mathbf \Sigma'$ and right singular vectors $\mathbf V'^\top$,
then reduces the rank to $\ell$.

\begin{algorithm}
	\caption{Frequent Directions Ridge Regression (\FDRR)}\label{fd-alg}
	\begin{algorithmic}[1]  
		\STATE {\bfseries Input:} $\ell, \bA, \bb, \gamma$
		\STATE $\mathbf \Sigma \gets 0^{\ell \times \ell}, \mathbf V^\top \gets 0^{\ell \times d}, \mathbf c \gets 0^d$
		\FOR{batches $(\bA_\ell, \bb_\ell) \in \bA, \bb$}
		\STATE $\_, \mathbf \Sigma', \mathbf V'^\top \gets \textsc{svd}([\mathbf V \mathbf \Sigma^\top; \bA_\ell^\top]^\top)$ \label{fd-alg:fd1}
		\STATE $\mathbf \Sigma \gets \sqrt{\mathbf \Sigma'^2_{\ell} - \sigma_{\ell+1}^2 \mathbf I_\ell}, \mathbf V \gets \mathbf V'_{\ell}$ \label{fd-alg:fd2}
		\STATE $\mathbf c \gets \mathbf c + \bA_\ell^\top \bb_\ell$ \label{fd-alg:c}
		\ENDFOR
		\STATE $\mathbf c' \gets \mathbf V^\top \mathbf c$ \label{fd-alg:x1}
		\STATE $\hat{\bx}_\gamma \hspace{-2pt} \gets \hspace{-1pt} \mathbf V (\mathbf \Sigma^2 + \gamma \mathbf I_\ell)^{-1} \mathbf c' + \gamma^{-1} (\mathbf c - \hspace{-1pt} \mathbf V \mathbf c')$ \label{fd-alg:x2}
		\STATE \textbf{return} $\hat{\bx}_\gamma$
	\end{algorithmic}
\end{algorithm}

Line 8 and 9 are how we compute the solution $\hat{\bx}_\gamma = (\mathbf V \mathbf \Sigma^2 \mathbf V^\top + \gamma \bI)^{-1} \mathbf c$. Explicitly inverting that matrix is not only expensive but also would use $O(d \times d)$ space, which exceeds the space limitation $O(\ell \times d)$. The good news is that $\mathbf V$ contains the eigenvectors of $(\mathbf V \mathbf \Sigma^2 \mathbf V^\top + \gamma \bI)^{-1}$, the corresponding $\ell$ eigenvalues $(\sigma_i^2 + \gamma)^{-1}$ for $i \in \{1, ..., \ell\}$, and the remaining eigenvalues are $\gamma^{-1}$. So we can separately compute $\hat{\bx}_\gamma$ in the subspace spanned by $\mathbf V$ and its null space.

\begin{thm}
\label{thm:fd}
Given $\bA \in \mathbb R^{n \times d}$, $\bb \in \mathbb R^n$,
let $\bx_\gamma = (\bA^\top \bA + \gamma \bI)^{-1} \bA^\top \bb$ and
$\hat{\bx}_\gamma$ be the output of Algorithm~\ref{fd-alg}
$\textup{\FDRR}(\ell, \bA, \bb, \gamma)$.
If
\[
\ell \ge \frac{\|\bA - \bA_k\|_F^2}{\gamma\eps} + k, 
\quad \text{ or } \quad
\gamma \ge \frac{\|\bA - \bA_k\|_F^2}{\eps(\ell - k)},
\]
then
\[
\|\hat{\bx}_\gamma - \bx_\gamma\| \le \eps \|\bx_\gamma\|.  
\]
It also holds that 
$\|\langle\hat{\bx}_\gamma, \mathbf a'\rangle - \langle\bx_\gamma, \mathbf a'\rangle\| \le \eps \|\bx_\gamma\|\|\mathbf a'\|$
for any $\mathbf a' \in \mathbb R^d$, and $\|\bA'\hat{\bx}_\gamma - \bA' \bx_\gamma\| \le \eps \|\bx_\gamma\|\|\bA'\|_2$
for any $\bA' \in \mathbb R^{m \times d}$. 
The squared statistical bias $\cB^2 (\hat \bx_\gamma) \le \left(1 + \frac{\eps^2}{\gamma^2}\|\bA\|_2^4 \right) \cB^2(\bx_\gamma)$,
and the statistical variance $\cV (\hat \bx_\gamma) \le (1+\|\bA\|_2^2/\gamma)^2 \cV(\bx_\gamma)$.
The running time is  $O(n \ell d)$ and requires space $O(\ell d)$.
\end{thm}

\begin{proof}
	Line 6 computes $\mathbf c = \bA^\top \bb$ in time $O(nd)$ using space $O(\ell d)$. Thus
	$
	\bx_\gamma = (\bA^\top \bA + \gamma \bI)^{-1} \mathbf c.  
	$
	
	Line 8 and 9 compute the solution $\hat{\bx}_\gamma = \mathbf V (\mathbf \Sigma^2 + \gamma \mathbf I_\ell)^{-1} \mathbf V^\top \mathbf c + \gamma^{-1} (\mathbf c - \mathbf V \mathbf V^\top \mathbf c)$ in time $O(d \ell)$ using space $O(d \ell)$. Let $\mathbf N \in \mathbb R^{d\times (d-\ell)}$ be a set of orthonormal basis of the null space of $\mathbf V$. Then
	\begin{align*}
	&\left( \mathbf V \mathbf \Sigma^2 \mathbf V^\top + \gamma \bI \right)^{-1} \\
	& = 
	\left(
	\left[\begin{array}{cc} \mathbf V & \mathbf N\end{array}\right]
	\left[\begin{array}{cc} \mathbf \Sigma^2+\gamma\mathbf I_\ell & 0 \\
	  0 & \gamma\mathbf I_{d-\ell} \end{array}\right]
	\left[\begin{array}{cc} \mathbf V & \mathbf N\end{array}\right]^\top
	\right)^{-1} 
	\\ &= 
	\left[\begin{array}{cc} \mathbf V & \mathbf N \end{array}\right]
	\left[\begin{array}{cc} (\mathbf \Sigma^2 + \gamma \mathbf I_\ell)^{-1} & 0 \\
	  0 & \gamma^{-1} \mathbf I_{d-\ell} \end{array}\right]
	\left[\begin{array}{cc} \mathbf V & \mathbf N\end{array}\right]^\top 
	\\&= 
	\mathbf V \left(\mathbf \Sigma^2 + \gamma \mathbf I_\ell \right)^{-1} \mathbf V^\top + \mathbf N \left(\gamma^{-1} \mathbf I_{d-\ell} \right) \mathbf N^\top 
	\\&= 
	\mathbf V \left(\mathbf \Sigma^2 + \gamma \mathbf I_\ell \right)^{-1} \mathbf V^\top + \gamma^{-1} \mathbf N \mathbf N^\top 
	\\ &= 
	\mathbf V \left(\mathbf \Sigma^2 + \gamma \mathbf I_\ell \right)^{-1} \mathbf V^\top + \gamma^{-1} \left(\bI - \mathbf V \mathbf V^\top \right).
	\end{align*}
	Thus
	$
	\hat{\bx}_\gamma = (\mathbf V \mathbf \Sigma^2 \mathbf V^\top + \gamma \bI)^{-1} \mathbf c.  
	$
	
	The rest of Algorithm~\ref{fd-alg} is equivalent to a normal FD algorithm with $\mathbf B = \mathbf \Sigma \mathbf V^\top$. Thus
	$
	\hat{\bx}_\gamma = (\mathbf B^\top \mathbf B + \gamma \bI)^{-1} \mathbf c,  
	$
	and satisfies (\ref{eq:fd-err}). 
	Together with Lemma~\ref{thm:general} and $\lambda_{min} (\mathbf B^\top \mathbf B) \ge 0$, we have
	\[
	\|\hat{\bx}_\gamma - \bx_\gamma\| \le \frac{\|\bA^\top \bA - \mathbf B^\top \mathbf B\|_2}{\lambda_{min} (\mathbf B^\top \mathbf B) + \gamma} \|\bx_\gamma\| \le \frac{\|\bA - \bA_k\|_F^2}{\gamma(\ell - k)} \|\bx_\gamma\|.  
	\]
	By setting $\frac{\|\bA - \bA_k\|_F^2}{\gamma(\ell - k)} = \eps$ and
	solving $\ell$ or $\gamma$, we get the guarantee for coefficients error.
	Plugging the FD result (\ref{eq:fd-err}) into Lemma \ref{thm:general_risk} gives us the risk bound.
	The running time and required space of a FD algorithm is $O(n \ell d)$ and $O(\ell d)$.
	Therefore the total running time is $O(nd) + O(\ell d) + O(n \ell d) = O(n\ell d)$,
	and the running space is $O(\ell d) + O(\ell d) + O(\ell d) = O(\ell d)$.
\end{proof}

\paragraph{Interpretation of bounds.}
Note that the only two approximations in the analysis of Theorem \ref{thm:fd}
arise from Lemma \ref{thm:general} and in the Frequent Directions bound.  
Both bounds are individually tight (see Lemma \ref{thm:general_tight},
and Theorem 4.1 by \citet{GhashamiSICOMP16FD}), so while this is not a complete
lower bound, it indicates this analysis approach cannot be asymptotically improved.  

We can also write the space directly for this algorithm to achieve
$\|\hat \bx - \bx_\gamma\| \leq \eps \|\bx_\gamma\|$ as
$O(d(k + \frac{1}{\eps} \frac{\|\bA - \bA_k\|_F^2}{\gamma}))$.  
Note that this holds for all choices of $k < \ell$, so the space is actually
$O(d \cdot \min_{0 < k < \ell} (k + \frac{1}{\eps} \frac{\|\bA - \bA_k\|_F^2}{\gamma}))$.  
So when $\gamma = \Omega(\|\bA - \bA_k\|_F^2)$ (for an identified best choice of $k$)
then this uses $O(d(k+\frac{1}{\eps}))$ space, and if this holds for a constant $k$,
then the space is $O(d/\eps)$.
This identifies the ``regularizer larger than tail'' case as when this
algorithm is in theory appropriate.
Empirically we will see below that it works well more generally.  

\subsection{Using Robust Frequent Directions}
If we use RFD instead of FD, we store $\alpha$ in addition to
$\mathbf B = \mathbf\Sigma \mathbf V^\top$;
see Algorithm \ref{rfd-alg}. Then the approximation of $\bA^\top \bA$ is
$
\mathbf B^\top \mathbf B + \alpha \bI = \mathbf V \mathbf \Sigma^2 \mathbf V^\top + \alpha \bI.  
$
We approximate $\bx_\gamma$ by
$
\hat{\bx}_\gamma = \left(\mathbf V \mathbf \Sigma^2 \mathbf V^\top + (\gamma + \alpha) \bI\right)^{-1} \mathbf c.  
$
Line 6 in Algorithm~\ref{rfd-alg} is added to maintain $\gamma + \alpha$.
The remainder of the algorithm is the same as Algorithm \ref{fd-alg}.
The theoretical results slightly improve those for FD.  
Theorem~\ref{thm:rfd} and its proof is established by replacing FD result with
RFD result (\ref{eq:rfd-err}) in Theorem~\ref{thm:fd}.

\begin{algorithm}
	\caption{Robust Frequent Directions Ridge Regression (RFDrr)}
	\label{rfd-alg}
	\begin{algorithmic}[1]  
		\STATE {\bfseries Input:} $\ell, \bA \in \mathbb, \bb, \gamma$
		\STATE $\mathbf \Sigma \gets 0^{\ell \times \ell}, \mathbf V^\top \gets 0^{\ell \times d}, \mathbf c \gets 0^d$
		\FOR{$\bA_\ell, \bb_\ell \in \bA, \bb$} 
		\STATE $\_, \mathbf \Sigma', \mathbf V'^\top \gets \textsc{svd}([\mathbf V \mathbf \Sigma^\top; \bA_\ell^\top]^\top)$ \label{rfd-alg:fd1}
		\STATE $\mathbf \Sigma \gets \sqrt{\mathbf \Sigma'^2_{\ell} - \sigma_{\ell+1}^2 \mathbf I_\ell}, \mathbf V \gets \mathbf V'_{\ell}$ \label{rfd-alg:fd2}
		\STATE $\gamma \gets \gamma + \sigma_{\ell+1}^2/2$ \label{rfd-alg:gamma}
		\STATE $\mathbf c \gets \mathbf c + \bA_\ell^\top \bb_\ell$ \label{rfd-alg:c}
		\ENDFOR
		\STATE $\mathbf c' \gets \mathbf V^\top \mathbf c$ \label{rfd-alg:x1}
		\STATE $\hat{\bx}_\gamma  \hspace{-2pt} \gets  \hspace{-1pt} \mathbf V (\mathbf \Sigma^2 + \gamma \mathbf I_\ell)^{-1} \mathbf c' + \gamma^{-1} (\mathbf c - \hspace{-1pt} \mathbf V \mathbf c')$ \label{rfd-alg:x2}
		\STATE \textbf{return} $\hat{\bx}_\gamma$
	\end{algorithmic}
\end{algorithm}

\begin{thm}
	\label{thm:rfd}
Given $\bA \in \mathbb R^{n \times d}$, $\bb \in \mathbb R^n$,
	let $\bx_\gamma = (\bA^\top \bA + \gamma \bI)^{-1} \bA^\top \bb$ and $\hat{\bx}_\gamma$ be output of Algorithm~\ref{rfd-alg} with input $(\ell, \bA, \bb, \gamma)$. If
	\[
	\ell \ge \frac{\|\bA - \bA_k\|_F^2}{2\gamma\eps} + k, 
	\quad \text{or} \quad 
	\gamma \ge \frac{\|\bA - \bA_k\|_F^2}{2\eps(\ell - k)}
	\]
	then 
	\[
    \|\hat{\bx}_\gamma - \bx_\gamma\| \le \eps \|\bx_\gamma\|
    \]
	It also holds that
	$\|\langle\hat{\bx}_\gamma, \mathbf a'\rangle - \langle\bx_\gamma, \mathbf a'\rangle\| \le \eps \|\bx_\gamma\|\|\mathbf a'\|$
	for any $\mathbf a' \in \mathbb R^d$, and $\|\bA'\hat{\bx}_\gamma - \bA' \bx_\gamma\| \le \eps \|\bx_\gamma\|\|\bA'\|_2$
	for any $\bA' \in \mathbb R^{m \times d}$. 
	The squared statistical bias $\cB^2 (\hat \bx_\gamma) \le \left(1 + \frac{4\eps^2} {\gamma^2}\|\bA\|_2^4 \right) \cB^2(\bx_\gamma)$,
	and the statistical variance $\cV (\hat \bx_\gamma) \le (1 + \|\bA\|_2^2/\gamma) \cV(\bx_\gamma)$.
	The running time is  $O(n \ell d)$ and requires space $O(\ell d)$.
\end{thm}

\section{\uppercase{Experiments}}

We compare new algorithms \FDRR and {\RFDRR} with other FD-based algorithms
and randomized algorithms on synthetic and real-world datasets.
We focus only on streaming algorithms.  

\paragraph{Competing algorithms} include:
\\
$\bullet$ \textsf{\iSVDRR}: Truncated incremental SVD~\citep{Brand02ISVD,Hall98IEfC},
also known as Sequential Karhunen–Loeve \citep{Levey00SKL}, for sketching,
has the same framework as Algorithm~\ref{fd-alg} but replaces Line 5
$\mathbf \Sigma \gets \sqrt{\mathbf \Sigma'^2_{\ell} - \sigma_{\ell+1}^2 \mathbf I_\ell}, \mathbf V \gets \mathbf V'_{\ell}$
with $\mathbf \Sigma \gets \mathbf \Sigma'_{\ell}, \mathbf V \gets \mathbf V'_{\ell}$.
That is, it simply maintains the best rank-$\ell$ approximation after each batch.  
\\
$\bullet$ \textsf{\NOFDRR:} This uses a two-level FD variant proposed by
\citet{HuangPMLR18} for sketching, and described in more detail in Section \ref{sec:FD}. 
\\
$\bullet$ \textsf{\RPRR:} This uses generic (scaled) \{-1,+1\} random projections~\citep{Sarlos06RP}.
For each batch of data, construct a random matrix $\mathbf S \in \{-\sqrt\ell, \sqrt\ell\}^{\ell \times \ell}$,
set $\mathbf C = \mathbf C + \mathbf S \bA$ and $\mathbf c = \mathbf c + \mathbf S \bb$.
Output $\hat{\bx}_\gamma = (\mathbf C^\top \mathbf C + \gamma \bI)^{-1} \mathbf C^\top \mathbf c$ at the end.
\\    
$\bullet$ \textsf{\CSRR:} This is the sparse version of \RPRR using the CountSketch~\citep{CW13}.
The random matrix $\mathbf S$ are all zeros except for one -1 or 1 in each column with a random location.
\\ 
$\bullet$ \textsf{\RR:} This is the naive streaming ridge regression which
computes $\bA^\top \bA$ and $\bA^\top \bb$ cumulatively (a batch size of $1$).
In each step it computes $\bA^\top \bA \leftarrow \bA^\top \bA + \ba_i^\top \ba_i$
where $\ba_i^\top \ba_i$ is an outer product of row vectors, and $\bc \leftarrow \bc + \ba_i^\top b_i$.
Then it outputs $\bx_\gamma = (\bA^\top \bA + \gamma \bI)^{-1} \bc$ at the end.
This algorithm uses $d^2$ space and has no error in $\bA^\top \bA$ or $\bc$.
This algorithm's found ridge coefficients $\bx_\gamma$ are used to compute the
coefficients error of all sketching algorithms.

\paragraph{Datasets.} 
We use three main datasets that all have dimension $d=2^{11}$,
training data size $n=2^{13}$, and test data size $n_t = 2^{11}$. 

\begin{figure}
	\centering
	\includegraphics{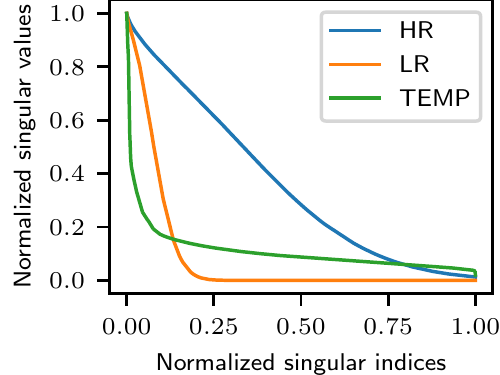}
	\caption{Datasets singular values}
	\label{fig:singulars}
\end{figure}

\emph{Synthetic datasets.} 
Two synthetic data-sets are low rank (\LR) and high rank (\HR),
determined by an effective rank parameter $R$;
set $R=\lfloor0.1 d\rfloor$ and $R = \lfloor0.5 d\rfloor$ respectively,
which is 10 and 50 percent of $d$.
This $R$ is then used as the number of non-zero coefficients $\bx$ and the
number of major standard deviations of a multivariate normal distribution for
generating input points $\bA$.
Each row vector of $\bA \in \mathbb R^{n \times d}$ are generated by normal
distribution with standard deviations $s_i = \exp(-\frac{i^2}{R^2})$ for $i=0,1,...,d-1$,
so the maximal standard deviation is $s_0 = 1$.
Figure~\ref{fig:singulars} shows the singular value distributions datasets,
normalized by their first singular values, and indices normalized by $d$.
The linear model coefficients $\bx \in \mathbb R^d$ have first $R$ entries non-zero,
they are generated by another standard normal distribution,
then normalized to a unit vector so the gradient of the linear model is 1.
A Gaussian noise $\iZ \sim \cN(\mathbf 0, 4\bI)$ is added to the outputs,
i.e. $\bb = \bA \bx + \mathbf \iZ$.
Finally, we rotate $\bA$ by a discrete cosine transform.

\begin{figure*}
    \centering
	\includegraphics{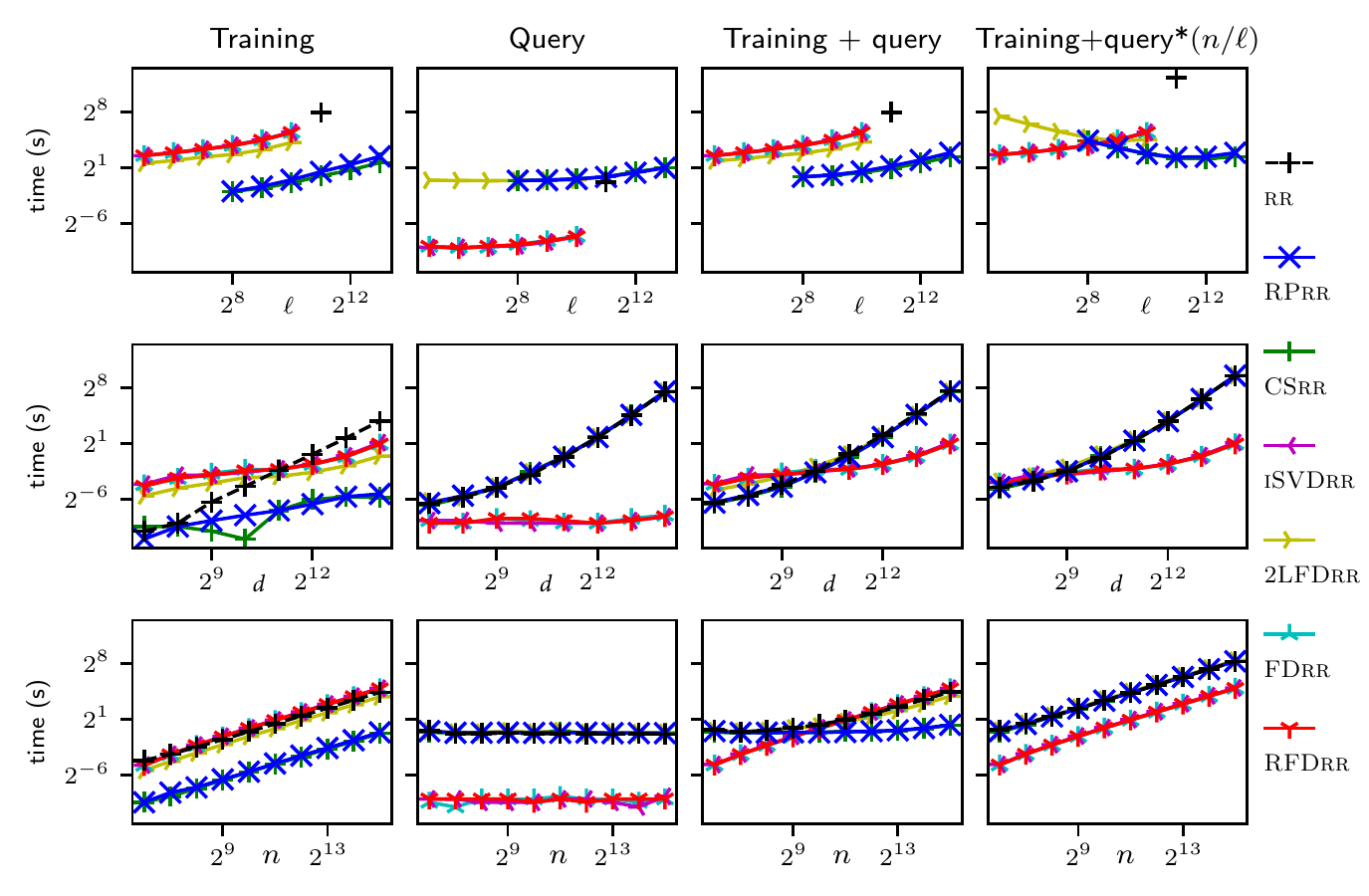}
    \vspace{-10pt}
	\caption{Running time (seconds) as a function of:
	           sketch size parameter $\ell$ (Row 1),  
	           data dimension $d$ (Row 2), and
	           training set size $n$ (Row 3).}
	\label{fig:runing-time}
\end{figure*}

\emph{\TEMP: Temperature sequence.}
This is derived from the temperature sequence recorded hourly from 1997 to 2019 at an 
international airport. 
To model an AR process, we compute the difference sequence between hourly temperatures,
and then shingle this data, so $\ba_i$ is $d$ consecutive differences starting
at the $i$th difference, and $b_i$ is the next (the $(i+d)$th) difference between temperatures.
Then the \TEMP dataset matrix $\bA$ is a set of $n$ randomly chosen (without replacement) such shingles.  

\paragraph{Choice of $\gamma$.}
We first run \RR on training datasets with different $\gamma$s,
then choose the ones which best minimize $\|\bA_\mathsf{test} x_\gamma^* - \bb_\mathsf{test}\|$
using a held out test dataset $(\bA_\mathsf{test},\bb_\mathsf{test})$.
The best $\gamma$s for low rank \LR and high rank \HR datasets are 4096 and
32768 respectively, the best $\gamma$ for \TEMP dataset is 32768.
These $\gamma$ values are fixed for the further experiments.
Since the $\gamma$ value is only used to compute the solution $\bx_\gamma$ or
$\hat \bx$ (storing $\alpha$ separate from $\gamma$ in \RFDRR), so this choice
could be made when calculating the solution using a stored test set after sketching.
To avoid this extra level of confounding error into the evaluation process,
we simply use this pre-computed $\gamma$ value.  

\subsection{Evaluation}

We run these 6 algorithms with different choices of $\ell$ on these three datasets.
They are implemented in python using numpy, and are relatively straightforward.
For completeness, we will release de-anonymized code and data for reproducibility after double-blind peer review.  
We first train them on the training sets, query their coefficients,
then compute the coefficients errors with \RR and prediction errors with outputs.
We repeat all these experiments 10 times and show the mean results.

\begin{figure*}
    \centering
	\includegraphics{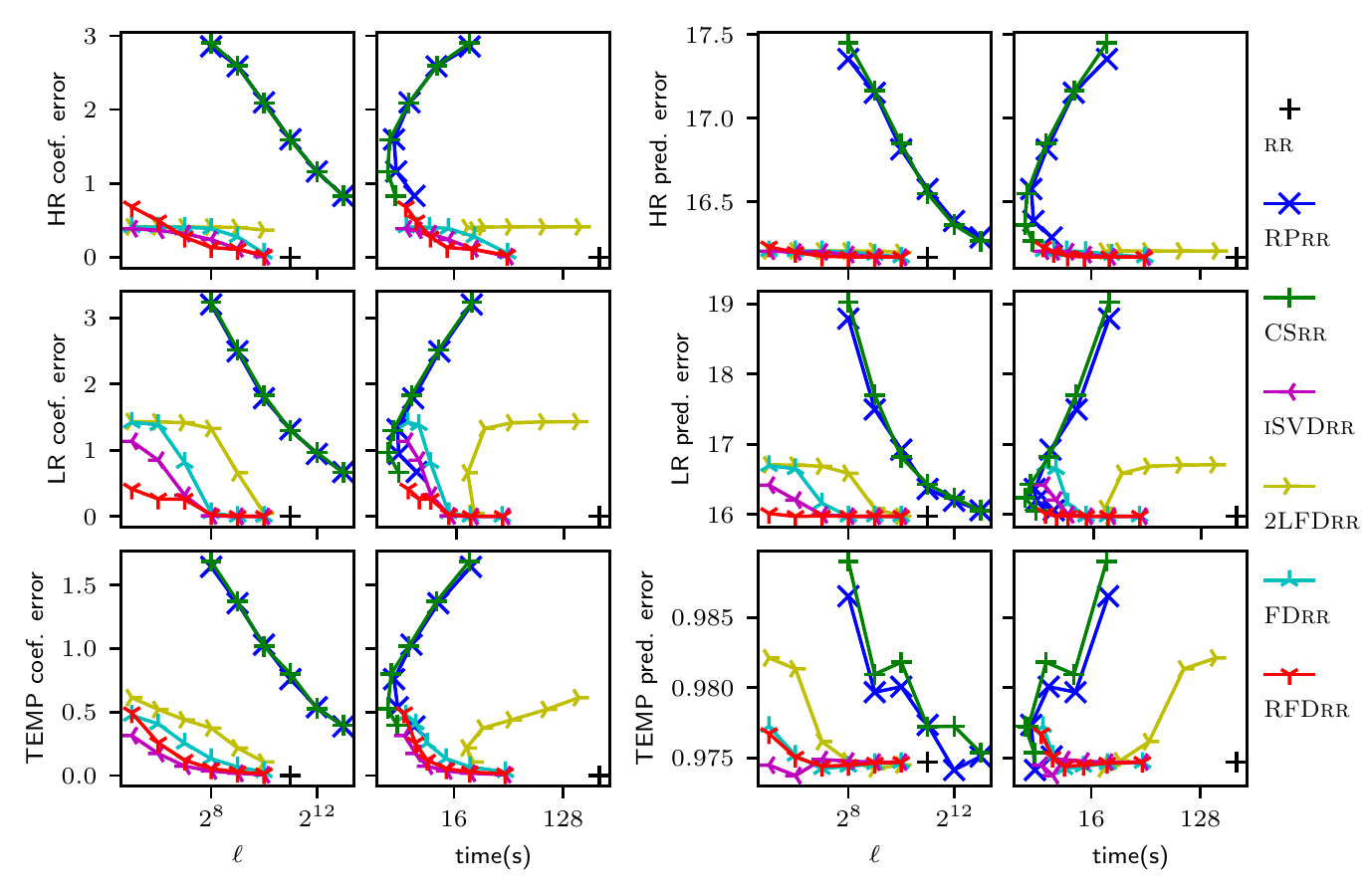}
    \vspace{-10pt}
	\caption{Errors vs space (measured by rows $\ell$) and time (measured by seconds).
	The time shown is the training time + the query time*$\frac{n}{\ell}$ to simulate a query every batch.
	The left double column shows coefficient error, and the right double column shows prediction error.
	Note that the runtime for \CSRR and \RPRR form a `C' shape since these are
	query-dominated, and the runtime initially decreases as the number of
	queries (number of ``batches'') decreases, as $\ell$ increases, like in
	Figure \ref{fig:runing-time}, Row 1.}
	\label{fig:errors}
\end{figure*}

\paragraph{Running time.}
In Figure~\ref{fig:runing-time}, Row 1 we show the running time (on HR) by
training time, solution query (computation of the coefficients) time, their sum,
and training time + query time$* n/\ell$ simulating making a query every batch.
The other datasets are the same size, and have the same runtimes.  
FD based algorithms are slower then randomized algorithms during training,
but much faster during query solutions since the sketch sizes are smaller and more processed. 
They maintain the SVD results of the sketch so the matrix inversion is mostly precomputed.  
Note that this precomputation is not available in the two-level \NOFDRR either,
hence this also suffers from higher query time.  

When we add the training time and $(n/\ell)$ queries, then \iSVDRR, \RFDRR,
and \FDRR are the fastest for $\ell$ below about $300$ (past $2^8$).
Note that in this plot the number of batches and hence queries decreases as $\ell$ increases,
and as a result for small $\ell$ the algorithms with cost dominated by queries
(\CSRR, \RPRR, and \NOFDRR) have their runtime initially decrease.  
All algorithms are generally faster than \RR {} -- the exception is the random
projection algorithms (\CSRR and \RPRR) which are a bit slower for query time,
and these become worse as $\ell$ becomes greater than $d$.

In Figure \ref{fig:runing-time}, Row 2 and 3 we show the runtime of the algorithms as both $n$ and $d$ increase. 
We fix $\ell = 2^6$.  When we vary $d$ we fix $n = 2^8$, and when we vary $n$ we fix $d = 2^{11}$.
As expected, the runtimes all scale linearly as $n$ grows, or the sum of two linear times for (training+query) time.
As $d$ grows, FD-based algorithms (not including \NOFDRR) overcome RP-based
algorithms (as well as \RR and \NOFDRR) even with one query.
The query time for the latter increase too fast, cubic on $d$, but is linear for FD-based algorithms.

\paragraph{Accuracy.}
Let $\bx_\gamma$ be the coefficients solutions of RR and $\hat{\bx}_\gamma$ be its approximation,
let $\hat{\bb}$ be the predicted values by RR, $\bA \bx_\gamma$, or its approximation,
$\bA \hat{\bx}_\gamma$;
for each algorithm we compute the coefficients error
($\textsf{\small coef. error} = \|\hat{\bx}_\gamma - \bx_\gamma\| / \|\bx_\gamma\|$)
and the prediction error ($\textsf{\small pred. error} = \|\hat{\bb} - \bb\|^2 / n$).
Figure~\ref{fig:errors} shows these errors versus space in terms of $\ell$,
and (training + $\frac{n}{\ell}$query) time in seconds. 
For the high rank data (top row), all FD-based algorithms
(\FDRR, \RFDRR, \NOFDRR, as well as \iSVDRR) have far less error than the
random projection algorithms (\RPRR and \CSRR).
For very small $\ell$ size \RFDRR does slightly worse than the other FD variants,
likely because it adds too much bias because the ``tail'' is too large with small $\ell$. 

For the low rank data and real-world \TEMP data the errors are more spread out,
yet the FD-based algorithms still do significantly better as a function of space ($\ell$).
Among these \RFDRR (almost) always has the least error (for small $\ell$) or
matches the best error (for larger $\ell$).
The only one that sometimes improves upon \RFDRR, and is otherwise the next
best is the huersistic \iSVDRR which has no guarantees, and likely will fail
for adversarial data~\citep{GhashamiTKDE16}.  
In terms of the time, the random projection algorithms can be a bit faster
(say $4$ seconds instead of $5-10$ seconds), but then achieve more coefficient error.
In particular, \RFDRR always can achieve the least coefficient error,
and usually the least coefficient error for any given allotment of time.  
For prediction error as a function of time (the rightmost column of Figure \ref{fig:errors}),
the results are more muddled.
Many algorithms can achieve the minimum error (nearly matching \RR) in the
nearly best runtime (about $5-7$ seconds).
The FD-based algorithms are roughly at this optimal points for all $\ell$
parameters tried above $\ell = 2^5$, and hence consistently achieves these
results in small space and time.  

\section{\uppercase{Conclusion \& Discussion}}

We provide the first streaming sketch algorithms that can apply the optimally
space efficient Frequent Directions sketch towards regression, focusing on ridge regression.
This results in the first streaming deterministic sketch using $o(d^2)$ space in $\R^d$.
We demonstrate that our bounds will be difficult to be improved, and likely cannot be.  
We also prove new risk bounds, comparable to previous results,
but notably have a variance bound independent of the specific sketch matrix chosen. 
Similar to prior observations~\citep{MM18,cohen_et_al:2016:ORS},
we show the ridge term makes regression easier to sketch.   
Moreover, our experiments demonstrate that while these FD-based algorithms have
larger training time than random projection ones, they have less empirical error,
their space usage is smaller, and query time is often far more efficient.
Our proposed sketches clearly have the best space/error trade-off.    

\paragraph{Discussion relating to PCR.}
Principal Component Regression (PCR) is a related approach;
it identifies the top $k$ principal components $\mathbf V_k$ of $\bA$ and performs regression using,
[$\pi_{\mathbf V_k}(\bA)$, $\bb$],  the projection onto the span of $\mathbf V_k$.
For this to be effective, these components must include the directions
meaningfully correlated with $\bA^\top \bb$.
However, when the top $k' > k$ singular vectors of $\bA$ are all similar,
which of the corresponding top $k'$ singular vectors are in the top $k$ is not stable.
If a meaningful direction among the top-$k$ is not retained in a top-$k$ sketch $\mathbf B$,
then while the norms of $\bA$ are preserved using a sketch $\mathbf B$,
the regression result may be quite different.
Hence, PCR is not stable in the same way as \RR, and precludes approximation
guarantees in the strong form similar to ours.  

\paragraph{Acknowledgements.}
Jeff M. Phillips thanks his support from NSF IIS-1816149, CCF-1350888, CNS-1514520, CNS-1564287, and CFS-1953350.

\bibliography{reference}

\onecolumn

\aistatstitle{A Deterministic Streaming Sketch for Ridge Regression \\
Supplementary Materials}

\thispagestyle{empty}

\appendix
\vfilneg

\section{\uppercase{Other Variance Bounds for Risk}}

We provide two different bounds for variance $\cV (\hat \bx_\gamma)$ that are
not strictly comparable with the one provided in Lemma \ref{thm:general_risk}.  

\begin{lem}\label{thm:var_bound_2}
	Considering the data generation model and the risk described in Lemma \ref{thm:general_risk}.
	The variance of the approximate solution
	$\hat \bx_\gamma = (\bC^\top \bC + \gamma \bI)^{-1} \bA^\top \bb$ satisfy
	\begin{align*}
	\cV (\hat \bx_\gamma) \leq \left( 1 + \frac1\gamma \|\bA\|_2^2 \|\bA^\top \bA - \bC^\top \bC\|_2^2 \|\bA^\dagger\|_2^2\right) \cV (\bx_\gamma)
	\end{align*}	
\end{lem}

\begin{proof}
\begin{align*}
    \cV (\hat \bx_\gamma)
    =& \mathbb E_\iZ \left[\|\bA \left(\hat \bx_\gamma - \mathbb E_\iZ \left[ \hat \bx_\gamma \right]\right)\|^2 \right]\\
    =& \mathbb E_\iZ \left[\|\bA \left((\bC^\top \bC + \gamma \bI)^{-1} \bA^\top s \iZ\right)\|^2 \right]\\
    =& s^2\|\bA (\bC^\top \bC + \gamma \bI)^{-1} \bA^\top\|_F^2\\
    =& s^2\|\bA \left( (\hat\bK + \gamma \bI)^{-1} - (\bK + \gamma \bI)^{-1} + (\bK + \gamma \bI)^{-1} \right) \bA^\top\|_F^2\\
    =& s^2\|\bA \left( (\hat\bK + \gamma \bI)^{-1} (\bK - \hat\bK) (\bK + \gamma \bI)^{-1} + (\bK + \gamma \bI)^{-1} \right) \bA^\top\|_F^2\\
    =& s^2\|\bA \left( (\hat\bK + \gamma \bI)^{-1} (\bK - \hat\bK) + \bI \right) (\bK + \gamma \bI)^{-1} \bA^\top\|_F^2\\
    =& s^2\|\bA \left( (\hat\bK + \gamma \bI)^{-1} (\bK - \hat\bK) + \bI \right) \bA^+ \bA (\bK + \gamma \bI)^{-1} \bA^\top\|_F^2\\
    \le& s^2\|\bA \left( (\hat\bK + \gamma \bI)^{-1} (\bK - \hat\bK) + \bI \right) \bA^+\|_2^2 \|\bA (\bK + \gamma \bI)^{-1} \bA^\top\|_F^2\\
    =& \|\bA (\hat\bK + \gamma \bI)^{-1} (\bA^\top \bA - \bC^\top \bC)\bA^+ + \bI \|_2^2 \cV (\bx_\gamma)\\
    \le& \left( 1 + \frac1\gamma \|\bA\|_2^2 \|\bA^\top \bA - \bC^\top \bC\|_2^2 \|\bA^\dagger\|_2^2\right) \cV (\bx_\gamma)
\end{align*}
\end{proof}

\begin{lem}\label{thm:var_bound_3}
	Considering the data generation model and the risk described in Lemma \ref{thm:general_risk}.
	The variance of the approximate solution
	$\hat \bx_\gamma = (\bC^\top \bC + \gamma \bI)^{-1} \bA^\top \bb$ satisfy
	\begin{align*}
	\cV (\hat \bx_\gamma) \leq \frac{1}{1-\|\bA^{+}\|^2\|(\bC^\top \bC - \bA^\top \bA)\|_2} \cV (\bx_\gamma)
	\end{align*}	
\end{lem}

\begin{proof}
The proof Follows the strategy used by \citet{WGM18} for the Hessian Sketch variance bound.
\begin{align*}
    \cV (\hat \bx_\gamma)
    =& \mathbb E_\iZ \left[\|\bA \left(\hat \bx_\gamma - \mathbb E_\iZ \left[ \hat \bx_\gamma \right]\right)\|^2 \right]\\
    =& \mathbb E_\iZ \left[\|\bA \left((\bC^\top \bC + \gamma \bI)^{-1} \bA^\top s \iZ\right)\|^2 \right]\\
    =& s^2\|\bA (\bC^\top \bC + \gamma \bI)^{-1} \bA^\top\|_F^2\\
    =& s^2\|({\bA^{+}}^{\top}\bC^\top \bC \bA^{+} + \gamma (\bA^\top \bA)^{-1})^{-1}\|_F^2\\
    \le& \frac{1}{1-\|\bA^{+}\|^2\|(\bC^\top \bC - \bA^\top \bA)\|_2} s^2 \|(\bI + \gamma (\bA^\top \bA)^{-1})^{-1} \|_F^2\\
    =& \frac{1}{1-\|\bA^{+}\|^2\|(\bC^\top \bC - \bA^\top \bA)\|_2} s^2\|\bA (\bA^\top \bA+\gamma \bI)^{-1} \bA^\top\|_F^2\\
    =& \frac{1}{1-\|\bA^{+}\|^2\|(\bC^\top \bC - \bA^\top \bA)\|_2} \cV (\bx_\gamma).
\end{align*}
The inequality follows
\begin{align*}
    & \|{\bA^{+}}^{\top}\bC^\top \bC \bA^{+} - \bI\|_2 \\
    =& \|{\bA^{+}}^{\top}\bC^\top \bC \bA^{+} - {\bA^{+}}^{\top}\bA^\top \bA \bA^{+}\|_2\\
    =& \|{\bA^{+}}^{\top}(\bC^\top \bC - \bA^\top \bA) \bA^{+}\|_2\\
    \le& \|\bA^{+}\|^2\|(\bC^\top \bC - \bA^\top \bA)\|_2.
\end{align*}
\end{proof}

\vfill
\end{document}